\documentclass[10pt,twocolumn,letterpaper]{article}

\usepackage{cvpr}
\usepackage{times}
\usepackage{epsfig}
\usepackage{graphicx}
\usepackage{amsmath}
\usepackage{amssymb}
\usepackage{amsthm}
\usepackage{dsfont}

\usepackage[pagebackref=true,breaklinks=true,letterpaper=true,colorlinks,bookmarks=false]{hyperref}

\newtheorem{proposition}{Proposition}
\makeatletter
\newcommand\footnoteref[1]{\protected@xdef\@thefnmark{\ref{#1}}\@footnotemark}
\makeatother
\cvprfinalcopy 

\def\cvprPaperID{****} 

\ifcvprfinal\fi
\begin{document}

\title{OL\'E: Orthogonal Low-rank Embedding,\\A Plug and Play Geometric Loss for Deep Learning}

\author{
\begin{tabular}{cccc} 
Jos\'e Lezama$^{1}$\thanks{Corresponding author {\tt jlezama@fing.edu.uy}} 
& Qiang Qiu$^2$ & Pablo Mus\'e$^1$ & Guillermo Sapiro$^2$\\ [0.7em]
\multicolumn{2}{c}{ \begin{tabular}{c}$^1$Universidad de la Rep\'ublica\\ Uruguay\ \end{tabular}} & 
\multicolumn{2}{c}{\begin{tabular}{c}$^2$Duke University\\ USA\end{tabular}}\\[0.3em]
\end{tabular}
}


\maketitle

\begin{abstract}
Deep neural networks trained using a softmax layer at the top and the
cross-entropy loss are ubiquitous tools for image classification. Yet, this does
not naturally enforce intra-class similarity nor inter-class margin of the
learned deep representations. To simultaneously achieve these two goals,
different solutions have been proposed in the literature, such as the pairwise
or triplet losses. However, such solutions carry the extra task of selecting
pairs or triplets, and the extra computational burden of computing and learning
for many combinations of them. In this paper, we propose a plug-and-play loss
term for deep networks that explicitly reduces intra-class variance and enforces
inter-class margin simultaneously, in a simple and elegant geometric manner. For
each class, the deep features are collapsed into a learned linear subspace, or
union of them, and inter-class subspaces are pushed to be as orthogonal as
possible. Our proposed Orthogonal Low-rank Embedding (OL\'E) does not require
carefully crafting pairs or triplets of samples for training, and works
standalone as a classification loss, being the first reported deep metric
learning framework of its kind.  Because of the improved margin between features
of different classes, the resulting deep networks generalize better, are more
discriminative, and more robust. We demonstrate improved classification
performance in general object recognition, plugging the proposed loss term into
existing off-the-shelf architectures. In particular, we show the advantage of
the proposed loss in the small data/model scenario, and we significantly advance
the state-of-the-art on the Stanford STL-10 benchmark. 

\end{abstract}

\begin{figure*} [t!]
  \begin{center}
\begin{tabular}{cccc} 
 \includegraphics[width=.15\textwidth]{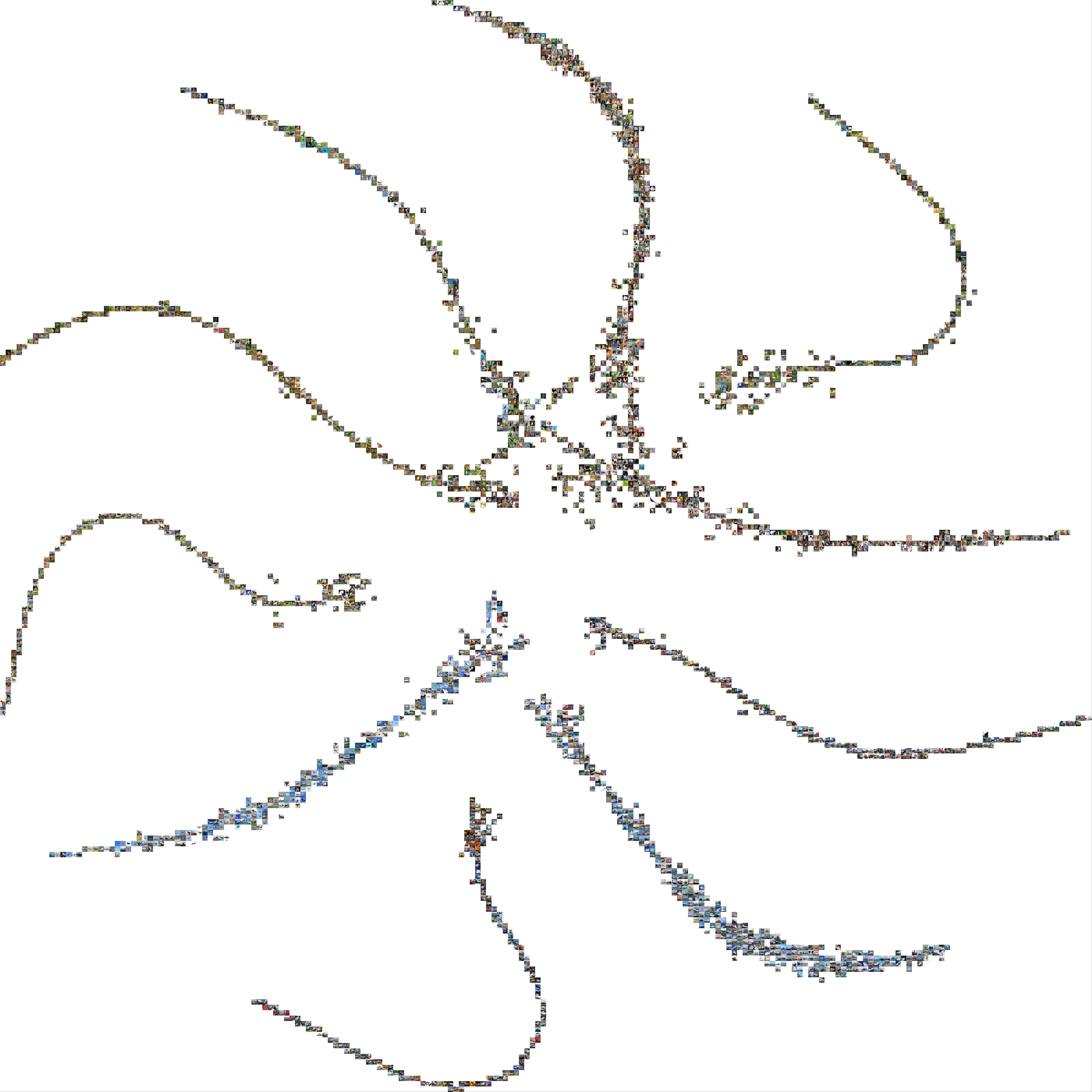} & %
 \includegraphics[width=.15\textwidth]{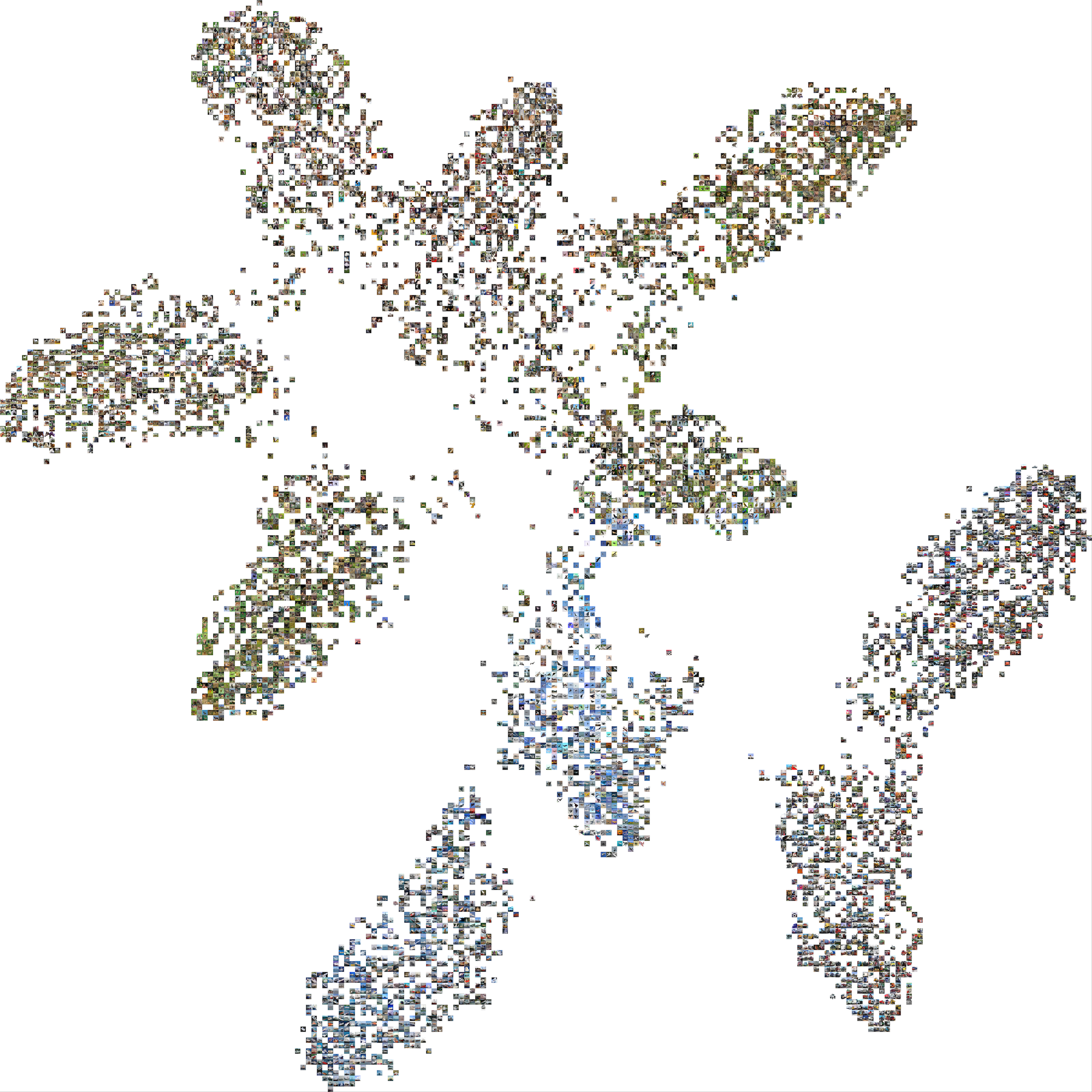} & %
\includegraphics[width=.20\textwidth]{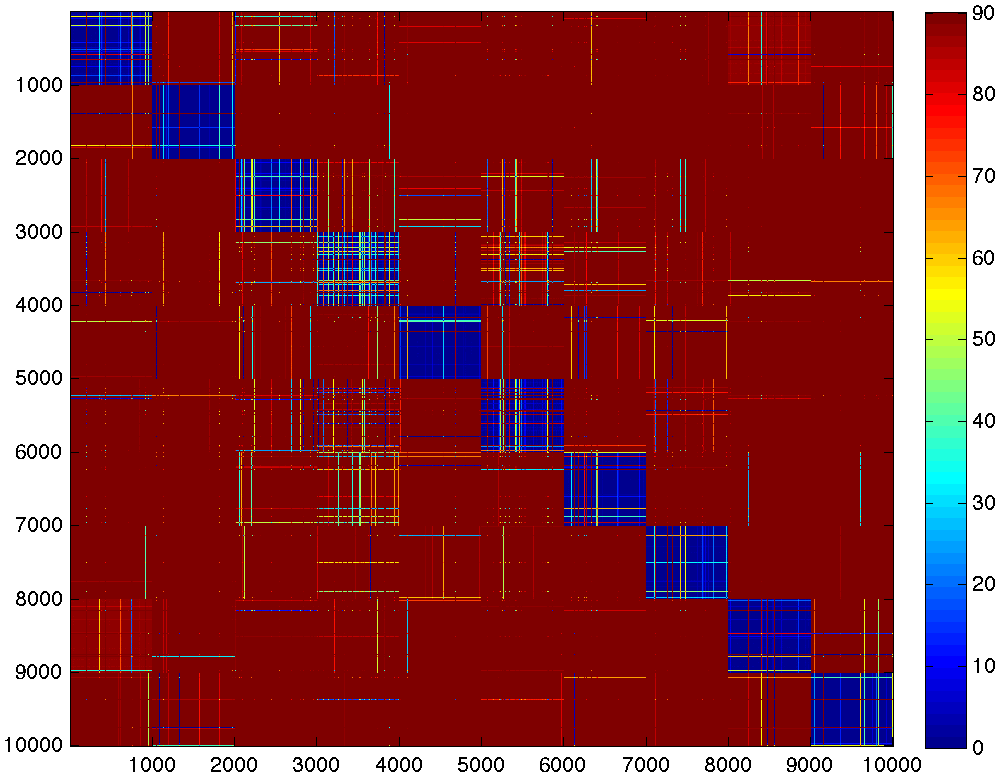} & %
\includegraphics[width=.20\textwidth]{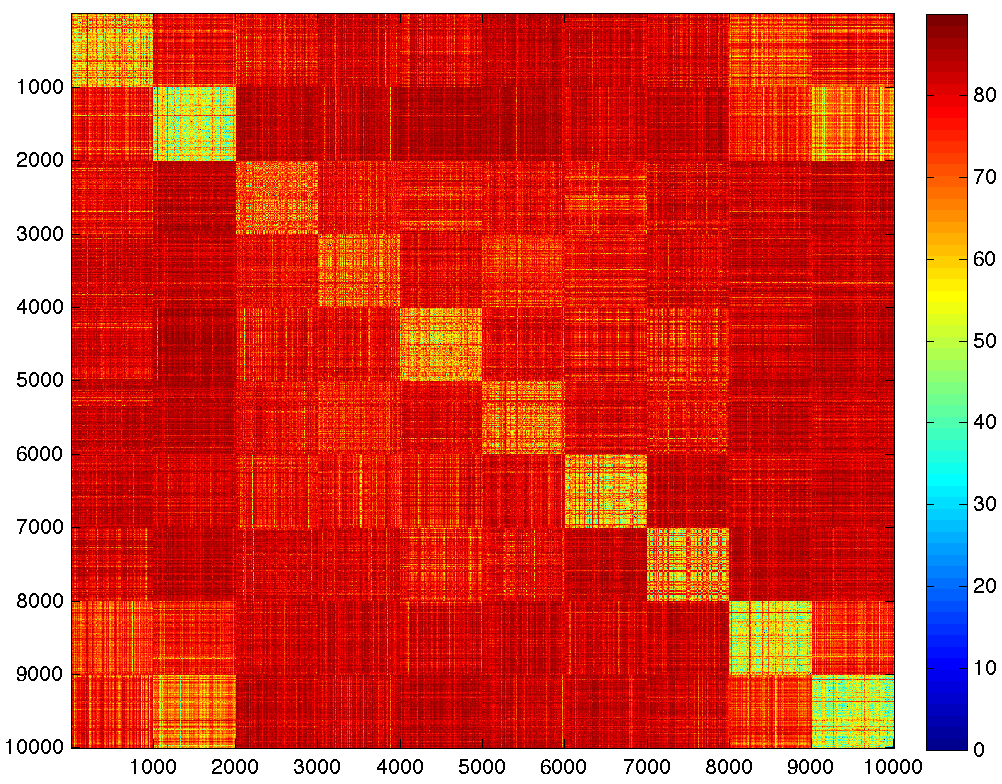}\\ %
\footnotesize{(a)} & \footnotesize{(b)} & \footnotesize{(c)} & \footnotesize{(d)}\\
\end{tabular} 
\end{center} 
  \caption{ Barnes-Hut-SNE \cite{van2013barnes} visualization of the deep
    feature embedding learned for the validation set of CIFAR10, using
    VGG-16. \textbf{(a)} With softmax loss and OL\'E . \textbf{(b)} With softmax
    loss only. The separation between classes is increased, and a low-rank
    structure is recovered for each class. \textbf{(c)} Angle between the features
    of the 10,000 validation samples, ordered by class, with OL\'E. \textbf{(d)}
    Without OL\'E.  With OL\'E the angle between features is collapsed inside
    each class and inter-class features are orthogonal. Best viewed in electronic format.}
\label{fig:cifar_tsne}
\end{figure*}

\begin{figure*} [t!]
  \begin{center}
\begin{tabular}{cc|cc}
\includegraphics[width=.16\textwidth]{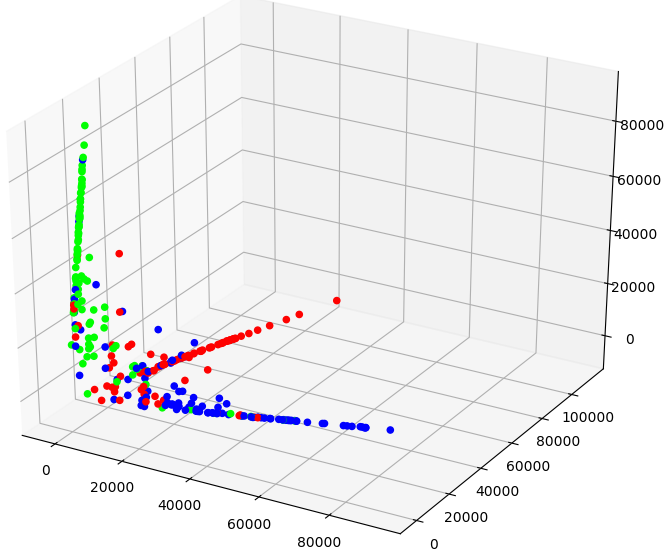}  & 
\includegraphics[width=.16\textwidth]{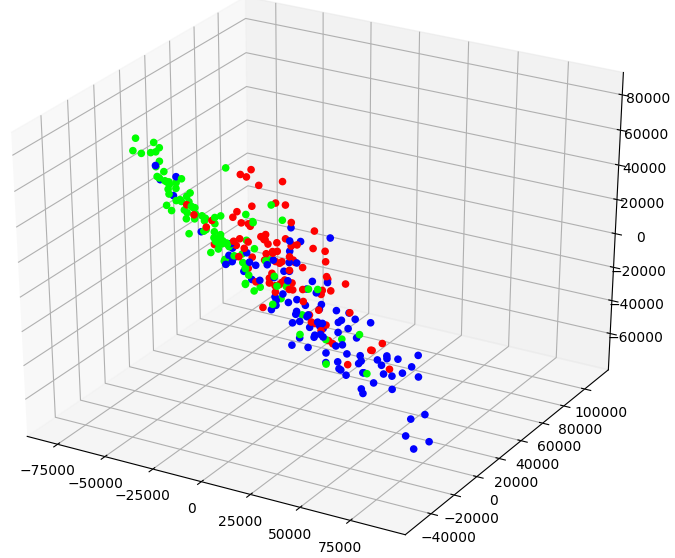} & 
\includegraphics[width=.16\textwidth]{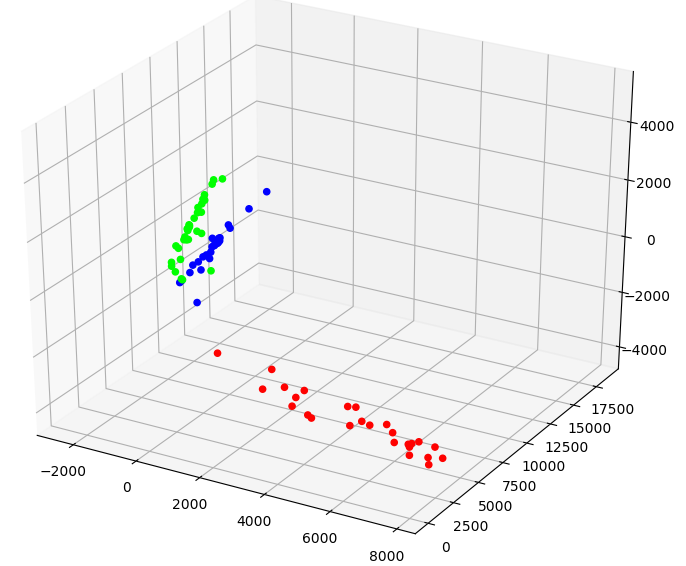} & 
\includegraphics[width=.16\textwidth]{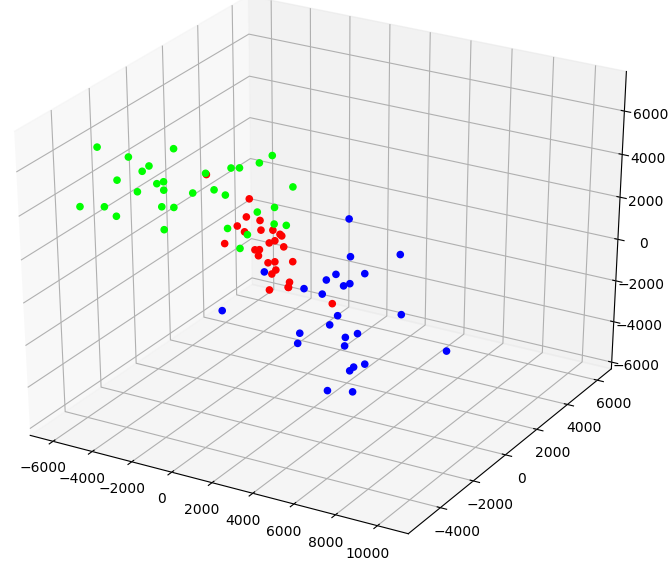} \\ 
\scriptsize{(a) OL\'E: 78.38\% accuracy} & %
\scriptsize{(b) Softmax:  76.69\% accuracy} &
\scriptsize{(c) OL\'E: 100.0\% accuracy} &
\scriptsize{(d) Softmax:  96.47\% accuracy} \\
\multicolumn{2}{c}{\scriptsize CIFAR 3 classes} & 
\multicolumn{2}{c}{\scriptsize Facescrub 3 classes}\\
\end{tabular}
\end{center} 
  \caption{Illustrative comparison between OL\'E loss (standalone) and softmax
    loss.  We show the actual 3D deep feature vectors for the validation images
    in two 3-class classification problems with scarce training data. OL\'E
    produces intra-class compactness and inter-class orthogonality, and is able to
    achieve better classification performance than the softmax loss.
    \textbf{(a)} \& \textbf{(b)} 3 classes of CIFAR10, trained with 1,000 images per
    class. A 4 layer, 100 hidden units MLP was used.  \textbf{(c)} \& \textbf{(d)} 3
    subjects of Facescrub dataset, trained with 110 images per class on
    average. A 3 layer, 10 hidden units MLP was used. See text for more
    details. Best viewed in electronic format.  }
\label{fig:toy_example}
\end{figure*}

\section{Introduction}
In the last few years the representational power of Deep Neural Networks (DNNs)
has been thoroughly demonstrated, with impressive results in learning useful
representations for difficult tasks such as object classification and detection
\cite{he2017mask, he2016deep,huang2016densely,
  krizhevsky2009learning,simonyan2014very} and face identification and verification
\cite{Parkhi15, schroff2015facenet,wen2016discriminative}, to name just a few
examples.  DNNs typically consist of a sequence of convolutional and/or
fully-connected layers with non-linear activation functions, which produce a
``deep'' feature vector, which is then classified in the last layer with a
linear classifier \cite{he2016deep, huang2016densely, krizhevsky2009learning,
  simonyan2014very}. This linear classifier typically uses the softmax function
with the cross-entropy loss. The combination of these two will be referred to as
\emph{softmax loss} in the rest of this article. The layer previous to the last linear classifier will
be referred to as the \emph{deep feature layer}.

Training a DNN with the standard softmax loss does not explicitly enforce an
embedding of the learned deep features where samples of the same class are
closer together and further away from other classes. To improve the
discrimination power of deep neural networks, previous approaches have tried to
enforce such an embedding via auxiliary supervisory loss functions
\cite{lee2015deeply} acting on the Euclidean distances between the deep features
\cite{cheng2016person,hadsell2006dimensionality,hu2014discriminative,
  schroff2015facenet, sun2014deep, wen2016discriminative}. Such metric learning
techniques are particularly popular in the face identification domain, with its
two most representative examples being the pairwise loss
\cite{hadsell2006dimensionality} and the triplet loss
\cite{schroff2015facenet}. The drawback with these approaches is that they
require the careful selection of pairs or triplets of samples, as well as extra
data processing. More recently, methods have been proposed to overcome this
limitation by enforcing intra-class compactness of the representations inside
each random training minibatch,
\cite{liu2017sphereface,liu2016large,wen2016discriminative}.

In this work we propose to improve the discriminability of a neural network by a
simple and elegant plug-and-play loss term that, acting on the deep feature
layer, encourages the learned deep features of the same class to lie in a linear
subspace (or union of them), and at the same time that inter-class subspaces are
orthogonal, see Figs.~\ref{fig:cifar_tsne},~\ref{fig:toy_example}. To the best of
our knowledge, this is the first time a deep learning framework is proposed that
simultaneously reduces intra-class variance and increases inter-class margin,
without requiring pair or triplet selection.

Our intuition is based on the following observations. First, that the decision
boundary for the softmax loss is determined by the angle between the feature
vector and the vectors corresponding to each class in the last linear classifier
\cite{liu2016large}. Since the weights are initialized randomly, the class
vectors are, with high probability, orthogonal at initialization, and typically
remain so after training. Moreover, if a rectified linear unit (ReLU) is the
last activation function, the deep features will live in the positive
orthant. Therefore, one way to improve the margin between deep features is to embed
them into orthogonal, low-dimensional linear subspaces, aligned with the
classifier vector of each class.

To this end, we adapt a shallow feature orthogonalization technique
\cite{qiu2015learning} to deep networks. Through novel theoretical insight, we
improve the objective formulation in \cite{qiu2015learning} and its
optimization. The outcome is a new loss function that can be plugged into any
existing deep architecture at the deep feature layer. We demonstrate via
thorough experimentation that this approach produces orthogonal deep
representations that lead to better generalization, not only in face
identification but also in general object recognition. We illustrate this on
different datasets and using four of the most popular CNN architectures: VGG
\cite{simonyan2014very}, ResNets \cite{he2016deep}, PreResNets \cite{ he2016identity} and DenseNets
\cite{huang2016densely}.

We demonstrate that our proposed technique is particularly successful in the
small data scenario. We significantly advance the state-of-the-art in the STL-10
standard benchmark \cite{coates2011analysis} when training with only 500 images
per class, and show that the advantage of OL\'E over the standard softmax loss
increases as fewer training samples are used.  We also show through a face
recognition application that because of the improved discriminability, the
network is better at detecting novel classes (outside the training set).

The source code for OL\'E is publicly available\footnote{\scriptsize \url{https://github.com/jlezama/OrthogonalLowrankEmbedding}}.

This paper is organized as follows. In Section~\ref{sec:related} we discuss
related work on similarity preserving and large margin deep networks. In
Section~\ref{sec:loss} we motivate and describe the orthogonalization loss and
its optimization, as well as a warming-up example. Experimental results are
presented in Section~\ref{sec:experimental}.  We conclude the paper in
Section~\ref{sec:conclusions}

\section{Related Work}\label{sec:related}
The first attempts to reduce the intra-class similarity of deep features and
increase their inter-class separation are metric learning based approaches
\cite{cheng2016person,hadsell2006dimensionality,hu2014discriminative,
  schroff2015facenet, sun2014deep, wen2016discriminative}. Their goal is to
minimize the Euclidean distance between the deep features of the same class,
while keeping the other classes apart. The pioneering contrastive loss
\cite{hadsell2006dimensionality} imposes such constraint using a siamese network
architecture \cite{chopra2005learning}. This pairwise strategy was particularly
popular in the face identification community \cite{hadsell2006dimensionality,
  sun2014deep, hu2014discriminative}, and was later extended to a triplet loss
\cite{schroff2015facenet, cheng2016person}.  With the triplet loss, an image
representation is simultaneously enforced to be close to a positive example of
the same class and far away from a negative example of a different class. The
main drawback of these approaches is that they require carefully mining for
pairs or triplets that effectively enforce the constraints.

A different strategy to encourage intra-class compactness
was proposed in \cite{wen2016discriminative}, where a centroid for the deep
representation of each class is updated in each training iteration, and
Euclidean distances to the centroids are penalized. This simple strategy produces compact
clusters for each class, although a large margin between clusters is not
explicitly enforced. Contrary to our method, the center loss cannot be
used standalone as a classification loss, since all the centroids tend to 
collapse to zero \cite{wen2016discriminative}. A related approach in
\cite{rippel2015metric} estimates a distribution  for the representation
of each class and penalizes class distribution overlap.

Based on the observation that the softmax loss is a function of the angles
between deep features and classifier vectors, a novel family of methods have been
proposed in \cite{liu2017sphereface, liu2016large}, that operate on such angles
and not on Euclidean distances. These works propose custom versions of the
softmax loss that encourage the features of one class to have a smaller angle
with their corresponding classification vector than in the standard softmax
loss. The improved margin produces notorious performance boosting with respect
to a standard network \cite{liu2017sphereface, liu2016large}.


In the unsupervised learning domain, a very recent family of methods proposes to
enforce a locally linear structure in the deep representations, such that
Subspace Clustering \cite{vidal2005generalized} can be later applied to the deep
representations \cite{ji2017deep, peng2017deep}.  These properties will arise
naturally in the deep representations learned with OL\'E, although imposed in a
supervised manner.

Our work is related to \cite{liu2017sphereface, liu2016large,
  wen2016discriminative}, in that we enforce intra-class compactness inside each
minibatch. However, for the first time, our objective function also
simultaneously encourages inter-class orthogonality, without the need to
carefully craft pairs or triplets.


This work stems from an orthogonalization technique used for shallow learning
proposed in \cite{qiu2015learning}. The orthogonalization is achieved via a
linear transformation enforcing a low-rank constraint on the features of the
same class, and a high-rank constraint on the matrix of features of all classes.
More precisely, consider a matrix $ \mathbf{Y} = \left[ \mathbf{y}_1 ~|~ \mathbf{y}_2 ~|~ \ldots ~|~ \mathbf{y}_N \right]$, where each column $\mathbf{y}_i \in
\mathbb{R}^d, i=1, \ldots ,N$ is a data point from one of the $C$ classes, and $|$
denotes horizontal concatenation. Let $\mathbf{Y}_c$ denote the submatrix formed
by the columns of $\mathbf{Y}$ that lie in the $c$-th class.  In
\cite{qiu2015learning}, a linear transform $ \mathbf{T}: \mathds{R}^d\rightarrow \mathds{R}^d$ is learned to minimize
\begin{align} \label{eq:nuclear_objective_original}
{ \sum_{c=1}^C ||\mathbf{T Y}_c||_* - ||\mathbf{T Y}||_*, ~~\mathrm{s.t.} ||\mathbf{T}||_2 = 1,}
\end{align}
where $||\mathbf{\cdot}||_*$ denotes the matrix nuclear norm, i.e., the sum of
the singular values of a matrix.  The nuclear norm is the convex envelope of the
rank function over the unit ball of matrices \cite{recht2010guaranteed}.  An
additional condition $||\mathbf{T}||_2 = 1$ is originally adopted to prevent the
trivial solution $\mathbf{T}=0$.

Here we adapt the loss in \eqref{eq:nuclear_objective_original} to the deep
learning framework and reformulate the loss and its optimization in a manner that
is suitable for training by backpropagation. 

\section{Orthogonalization Loss}\label{sec:loss}
\subsection{Motivation}
Consider a neural network whose last fully connected layer is $\mathbf{W} \in
\mathds{R}^{C\times D}$, where $D$ is the dimension of the deep features and $C$
the number of classes. Each row $\mathbf{w}_c \in \mathds{R}^D$ of $\mathbf{W}$
represents a linear classifier for class $c$. If $\mathbf{W}$ is initialized
randomly, then such rows are (with high probability) orthogonal. Now consider
$\mathbf{x}$ as the deep representation of an image (or any other data being
classified). If the activation function in the deep feature layer is the
element-wise maximum between $\mathbf{x}$ and $0$ (ReLU), then $\mathbf{x}$
always lives in the positive orthant. From these two observations it can be
deduced that at the end of a successful training of the network the classifier
vectors $\mathbf{x}_c$ should remain orthogonal to have the most separation
between classes.  Therefore, one strategy to learn large-margin deep features is
to make the intra-class features fall in a linear subspace aligned with the
corresponding classification vector, while features of different classes should
be orthogonal to each other.  This natural geometry of learned features is not
imposed by standard last-layer classifiers in today's leading architectures.

\subsection{Definition}
We propose to enforce the aforementioned orthogonalization by adapting
\eqref{eq:nuclear_objective_original} to the deep learning setting. Namely,
suppose for a given training minibatch $\mathbf{Y}$ of $N$ samples,
$\mathbf{X}=\Phi(\mathbf{Y};\mathbf{\theta})$ is the $N\times D$ deep embedding $\Phi$ of
the data, parameterized by $\mathbf{\theta}$.

Let $\mathbf{Y}_c$, $\mathbf{X}_c$ be the data and the sub-matrix of deep
features belonging to class $c$, respectively, and $\mathbf{X}$ the matrix of
deep features for the entire minibatch $\mathbf{Y}$. We propose the following
OL\'E loss:
\begin{align}\label{eq:deep_objective}
\mathbf{L}_{o}(\mathbf{X}) &:={ { \sum_{c=1}^C \max(\Delta, ||\mathbf{X}_c||_*) - ||\mathbf{X}||_*}}\\
&= { { \sum_{c=1}^C \max(\Delta, ||\Phi(\mathbf{Y}_c;\theta)||_*) - ||\Phi(\mathbf{Y};\theta)||_*}}
\end{align}

With respect to \eqref{eq:nuclear_objective_original}, we drop the linear
transformation $T$ (the network is already transforming the data) and its
normalization restriction, and we add a bound $\Delta \in \mathds{R}$ on the
intra-class nuclear loss, so that after a certain point the intra-class norm
reduction is no longer enforced, thus avoiding the collapse of the features to
zero  (and therefore no needing the normalization).  We will always use $\Delta=1$ for the experiments in this paper.

The global minimum of \eqref{eq:deep_objective} is reached when each of the
$\mathbf{X}_c$ matrices are orthogonal to each other \cite{qiu2015learning}.  We
next describe a simple descent direction for optimizing $\mathbf{\theta}$
\eqref{eq:deep_objective} via backpropagation, and show that this direction
vanishes only when the orthogonalization is achieved.

\subsection{Optimization}
In order to optimize \eqref{eq:deep_objective} via backpropagation, we need to
compute a subgradient of the nuclear norm of a matrix. Let $\mathbf{A}
=\mathbf{U}\mathbf{\Sigma}\mathbf{V}^T$ be the SVD decomposition of the $m\times
n$ matrix $\mathbf{A}$. Let $\delta$ be a small threshold value, and $s$ the
number of singular values of $\mathbf{A}$ larger than $\delta$. Let
$\mathbf{U}_1$ be the first $s$ columns of $\mathbf{U}$ and $\mathbf{V}_1$ be
the first $s$ columns of $\mathbf{V}$ (corresponding to those larger than
$\delta$ eigenvalues). Correspondingly, let $\mathbf{U}_2$ be the remaining
columns of $\mathbf{U}$ and $\mathbf{V}_2$ the remaining columns of
$\mathbf{V}$. Then, a subdifferential of the nuclear norm is
(\cite{recht2010guaranteed, watson1992characterization})
\begin{equation}
  \partial{||\mathbf{A}||_*} = \mathbf{U}_1\mathbf{V}_1^T + \mathbf{U}_2\mathbf{W}\mathbf{V}_2^T,
\end{equation}
with $||\mathbf{W}||\leq 1$.  

Here we propose to use $\mathbf{W}=0$, obtaining the following projected
subgradient for the nuclear norm minimization problem:
\begin{equation}
g_{||A||_*}(A) = \mathbf{U}_1\mathbf{V}_1^T.  
\end{equation}

Intuitively, to avoid numerical issues, we are dropping the directions of the
subgradient onto which the data matrix has no or very low energy already (i.e.,
their corresponding singular values are already close to 0).  This improves upon
the formulation  in \cite{qiu2015learning}, where all the directions were
used.

Suppose $\mathbf{X} = \left[ \mathbf{X}_1 ~|~ \mathbf{X}_2 ~|~ \ldots ~|~
  \mathbf{X}_C \right]$ is the deep feature matrix of one minibatch. For
$\mathbf{X}_c$, the feature submatrix of each class $c\in\{1, \ldots ,C\}$, let
$\mathbf{U}_{1c}$ and $\mathbf{V}_{1c}$ be its principal left and right singular
vectors. Let $\mathbf{U}_1$ and $\mathbf{V}_1$ be the principal left and right
singular vectors of $\mathbf{X}$, the deep feature matrix of all the classes
combined. (By principal we mean those whose corresponding singular value is
greater than the threshold $\delta$.) Then, we propose the following descent
direction for \eqref{eq:deep_objective}:

\begin{equation}\label{eq:descent_direction}
g_{L_o}(\mathbf{X}) := \sum_{c=1}^C \left[ \mathbf{Z}_c^{(l)} ~|~ \mathbf{U}_{c1}\mathbf{V}_{c1}^T  ~|~ \mathbf{Z}_c^{(r)} \right] - \mathbf{U}_1\mathbf{V}_1^T.
\end{equation}
Here, $\mathbf{Z}^{(l)}_c$ and $\mathbf{Z}^{(r)}_c$ are fill matrices of zeros
to complete the dimensions of $\mathbf{X}$. The first term in
\eqref{eq:descent_direction} reduces the variance of the principal components of
the per-class features. The second term increases the variance of all the
features together, projecting the feature matrix onto its closest orthogonal
form.

Next we prove that this direction vanishes only when the objective reaches the global
minimum of zero.

\begin{proposition}
  If $g_{L_o}(\mathbf{X}) = 0$ and $||X_c||_*>\Delta$, then $L_o(\mathbf{X}) = 0$.
\end{proposition}
\begin{proof}
We give the proof for two classes, its extension to multiple classes is
straightforward.  Let $\mathbf{X} = \left[\mathbf{A} ~|~ \mathbf{B} \right]$ with
$\mathbf{A}$ and $\mathbf{B}$ corresponding to the feature matrices of two
classes. Let $\mathbf{A} = \mathbf{U}_{A1}\mathbf{\Sigma}_{A1}\mathbf{V}_{A1} +
\mathbf{U}_{A2}\mathbf{\Sigma}_{A2}\mathbf{V}_{A2}$ and $\mathbf{B} =
\mathbf{U}_{B1}\mathbf{\Sigma}_{B1}\mathbf{V}_{B1} +
\mathbf{U}_{B2}\mathbf{\Sigma}_{B2}\mathbf{V}_{B2}$ be their SVD decomposition,
where the subscript $1$ corresponds to the singular values larger than the
threshold $\delta$ and the subscript $2$ to the remaining singular values.  Let
$\mathbf{0}$ be a generic matrix of  zeroes, whose size is determined by
context, for simplicity. Then,
\begin{align}
&L_o(\mathbf{X}) = ||A||_* + ||B||_* - ||[A~|~B]||_*, \\
&g_{L_o}(\mathbf{X}) = 
\left[ \begin{array}{c|c} \mathbf{U}_{A1}\mathbf{V}_{A1}^T & \mathbf{0} \end{array} \right] + 
\left[ \begin{array}{c|c} \mathbf{0} & \mathbf{U}_{B1}\mathbf{V}_{B1}^T \end{array} \right] - 
\mathbf{U}_1\mathbf{V}_1^T.
\end{align}
Then, $g_{L_o}(\mathbf{X}) = 0$ implies
\begin{align}
   \mathbf{U}_1\mathbf{V}_1^T &= 
\left[ \begin{array}{c|c} \mathbf{U}_{A1}\mathbf{V}_{A1}^T & \mathbf{0} \end{array} \right] + 
\left[ \begin{array}{c|c} \mathbf{0} & \mathbf{U}_{B1}\mathbf{V}_{B1}^T \end{array} \right]\\
&= \left[ \begin{array}{c|c} \mathbf{U}_{A1} & \mathbf{U}_{B1} \end{array} \right]
\left[\begin{array}{c|c} \mathbf{V}_{A1}^T & \mathbf{0}\\  \mathbf{0} & \mathbf{V}_{B1}^T \end{array} \right]. \label{eq:orthogonal_product}
\end{align}
Since $\mathbf{U}_1$ and $\mathbf{V_1}$ are orthogonal matrices, and the
rightmost matrix in \eqref{eq:orthogonal_product} is also orthogonal, then
$\left[\mathbf{U}_{A1} ~|~ \mathbf{U}_{B1} \right]$ must  be orthogonal.  Since $\mathbf{U}_{A1}$ and $\mathbf{U}_{B1}$
are orthogonal submatrices, this implies that their columns must be orthogonal
to each other.  Then, $\mathbf{A}$ and $\mathbf{B}$ are orthogonal to each other
and thus $L_o(\mathbf{X})=0$ \cite{qiu2015learning}~(Theorem 2).

\end{proof}

In practice, we observe empirical convergence in all our
experiments. Fig.~\ref{fig:convergence} shows typical convergence curves for the
OL\'E loss when used standalone or in combination with the softmax loss.

\subsection{Illustrative  Example}
In Fig.~\ref{fig:toy_example} we show via two simple illustrative examples the
result of applying the OL\'E loss of \eqref{eq:deep_objective} as the objective
function of a neural network. We compare with the result of applying the
traditional softmax loss.

For the first experiment, we used 3 classes from CIFAR10 (0: plane, 1: car,
2: bird) and trained a Multi-Layer Perceptron (MLP), with 4 hidden
layers of 100 neurons each, and a final layer of dimension 3. The network was
trained for 300 epochs on 1,000 images per class and evaluated in 100 images per class.

For the second experiment, we used 3 randomly chosen subject identities from the
Facescrub dataset \cite{ng2014data}: Al Pacino, Helen Hunt, and Sean Bean. Each
identity contains, on average, 110 images for training and 30 for validation. We
used a 3 layer MLP with 10 neurons in each hidden layer, and trained for 150
epochs. All the MLPs use ReLU activation functions, batch normalization and weight decay
and were trained with Adam with learning rate $10^{-4}$. 

For the comparison, the architecture and hyperparameters are shared and only the
objective function is changed. For evaluation, we use 1-Nearest-Neighbor with
cosine distance, (this yielded equivalent or better performance than using the
softmax score). We ran the training 50 times for each architecture and dataset
and kept the model giving the best classification result in the validation set.

In Fig.~\ref{fig:toy_example} we plot the actual 3D deep feature vectors
obtained by the networks for the validation set. We observe a successful
orthogonalization of the learned features when using OL\'E and a better
classification performance, in particular for the Facescrub experiments, where
the number of samples per class is very limited.

In the following section, we will combine the power of the OL\'E and the
softmax loss to achieve significant performance gains, in particular in the
small data scenario.

\subsection{Discussion}
The proposed embedding  has several advantages with
respect to similar  embeddings in the literature:
\begin{itemize}
\item It does not require carefully crafting  pairs or triplets of samples, and
  works simply as a plug-and-play loss that can be appended to any existing
  network architecture.
\item Compared to the Large-Margin Softmax Loss in \cite{liu2016large} or the
  A-softmax loss in \cite{liu2017sphereface}, the OL\'E loss is not restricted
  to be used with a softmax classifier and can be used standalone or as a
  complement of any other loss, or to impose orthogonality at any layer of the
  network.
\item Compared to the Center Loss of \cite{wen2016discriminative}, our deep
  objective function encourages intra-class compactness and inter-class
  separation simultaneously, whereas \cite{wen2016discriminative} does only the
  former. Also, the Center Loss cannot be used standalone.
\item OL\'E collapses the deep features into linear subspaces. When
  used in conjunction with the softmax loss, the linear classifiers of the last
  layer  find a natural form which is a vector aligned with the linear
  subspace.
\end{itemize}

\section{Experimental Evaluation}\label{sec:experimental}
In this section, we demonstrate the improved generalization performance obtained
when using the OL\'E loss in combination with the standard softmax loss for
several popular deep network architectures and different standard visual
classification datasets. We will also further analyze the effect of the proposed
embedding.

In all experiments, we seek to minimize the
combination of the softmax classification loss and the OL\'E loss:
\begin{equation}\label{eq:combined_objective}
\underset{\mathbf{\theta}}{\min} ~ {L_s}(\mathbf{X},\mathbf{y},\mathbf{\theta}) 
+ \lambda \cdot {L_o}(\mathbf{X},\mathbf{y},\mathbf{\theta}^*)  
+ \mu \cdot ||\mathbf{\theta}||^2,
\end{equation}
where $L_S$ is the standard softmax loss (softmax layer plus cross-entropy
loss). The second term $L_o$ is the proposed OL\'E loss \eqref{eq:deep_objective}. The
parameter $\lambda$ controls the weight of the OL\'E loss; $\lambda=0$
corresponds to standard network training. 

Here $\mathbf{\theta}^*$ means every weight in the network except the weights of
the last fully-connected layer, which is the linear classifier. This is because the
OL\'E loss is applied to the deep features at the penultimate layer. The third
term represents the standard weight decay. The values used for these parameters are
detailed below.

\subsection{Datasets}
\textbf{SVHN.} The Street View House Numbers (SVHN) dataset
\cite{netzer2011reading} contains $32 \times 32$ colored images of digits 0 to
9, with $73,257$ images for training and 26,032 for testing. We did not use
the additional unlabeled training images nor performed any data augmentation.

\textbf{MNIST.} The MNIST database contains $28\times 28$ grayscale images of
digits from 0 to 9. The training and testing set contain 60,000 and 10,000
examples respectively. No data augmentation was used.

\textbf{CIFAR10 and CIFAR100.} The two CIFAR datasets
\cite{krizhevsky2009learning} contain $32 \times 32$ colored images from 10 and
100  object classes respectively. Both datasets contain 50,000 images for
training and 10,000 for testing.  When using data augmentation, we append the
suffix '+' to the dataset name. We used the standard data augmentation for
CIFAR: 4 pixel padding, $32 \times 32$ random cropping and horizontal flipping.

\textbf{STL-10.} The Self-Taught Learning 10 (STL-10) dataset
\cite{coates2011analysis} contains $96 \times 96$ colored images from 10 object
categories. Designed for semi-supervised and unsupervised learning, there are
only $500$ training images and $800$ test images with labels per class. 
Data augmentation consisted of 12 pixel padding, and random $96 \times 96$
cropping and horizontal flipping. We add the '+' suffix when reporting results
using data augmentation.

\begin{table*}[t]
\centering
\begin{footnotesize}
\begin{tabular}{|l|c|}
\hline
VGG-11 & \textbf{C}64-\textbf{MP}-\textbf{C}128-\textbf{MP}-\textbf{C}256(x2)-\textbf{MP}-\textbf{C}512(x2)-\textbf{MP}-\textbf{C}512(x2)-\textbf{MP}-\textbf{FC}512  \\
\hline
VGG-16  & \textbf{C}64(x5)-\textbf{MP}-\textbf{C}128(x4)-\textbf{MP}-\textbf{C}256(x4)-\textbf{MP}-\textbf{FC}256 \\
\hline
VGG-19 & \textbf{C}64(x2)-\textbf{MP}-\textbf{C}128(x2)-\textbf{MP}-\textbf{C}256(x4)-\textbf{MP}-\textbf{C}512(x4)-\textbf{MP}-\textbf{C}512(x4)-\textbf{MP}-\textbf{FC}512 \\
\hline
VGG-FACE & \textbf{C}64(x2)-\textbf{MP}-\textbf{C}128(x2)-\textbf{MP}-\textbf{C}256(x3)-\textbf{MP}-\textbf{C}512(x3)-\textbf{MP}-\textbf{C}512(x3)-\textbf{FCD}4096(x2)-\textbf{FC}1024\\
\hline
ResNet-110 &   \textbf{C}16-\textbf{R}64/16(x18)-\textbf{R}128/32(x18)-\textbf{R}256/64(x18)-\textbf{AP} \\
\hline
Pre-ResNet-110 &  \textbf{C}16-\textbf{PR}64/16(x18)-\textbf{PR}128/32(x18)-\textbf{PR}256/64(x18)-\textbf{BN}-\textbf{ReLU}-\textbf{AP} \\
\hline
DenseNet-40-12 & \textbf{CO}24-\textbf{D}168/12-\textbf{CR}168-\textbf{D}312/12-\textbf{CR}312-\textbf{D}456/12-\textbf{BN}456-\textbf{ReLU}\\
\hline
CNN-5 & \textbf{C}32-\textbf{MP}-\textbf{C}64-\textbf{MP}-\textbf{C}128-\textbf{MP}-\textbf{C}256-\textbf{C}256-\textbf{MP} \\
\hline
\end{tabular}
\end{footnotesize}
\caption{Summary of the deep network architectures used in our experiments. The
  last Fully-Connected layer, whose size depends on the number of classes used,
  is not shown. \textbf{C}X: Convolutional block. Kernel size is always is
  3x3. \textbf{MP:} Max pooling with kernel size 2x2 and stride 2. \textbf{FC}X:
  Fully-Connected layer. \textbf{R}X/Y and \textbf{PR}X/Y: ResNet and PreResNet
  Blocks Respectively. \textbf{AP}: Global Average Pooling layer. \textbf{CO}X:
  Plain convolutional layer. \textbf{D}X/G: DenseNet Block. \textbf{PC}X: Pre-BN
  convolutional block: (BN-conv.-ReLU). Kernel size is 1x1.  For all the modules, X
  is the number of output channels. Y is the number of inner channels for \textbf{R}
  and \textbf{PR} blocks and G is the growth rate for \textbf{D} blocks. See text
  for detailed block definitions. The OL\'E loss is always applied at the output
  of the last layer shown in this table.}
\label{tab:architectures}
\end{table*}

\textbf{Facescrub-500.} The Facescrub-500 dataset is obtained by selecting 500
of the 530 identities of the Facescrub dataset \cite{ng2014data}. The remaining
30 classes were used for evaluating out of sample performance.  We split the
images of the first 500 subjects into a training and a testing datasets with on
average $91$ images for training and $23$ images for testing per class
(80\%/20\% split).  We preprocess the images by aligning facial landmarks using
\cite{landmark} and crop the resulting aligned face images to $224 \times224$,
with color.

\subsection{Network Architectures}
Evaluated architectures are summarized in Table~\ref{tab:architectures}.

\textbf{VGG.} The VGG architecture \cite{simonyan2014very} consists of
blocks of convolutional layers with ReLU activation functions and Batch
Normalization (BN), linked by Max-Pooling layers and with one or more
fully-connected (FC) layers at the end. For VGG-11 and VGG-19 we use a publicly
available
implementation\footnote{\label{bearpaw}\scriptsize{\url{https://github.com/bearpaw/pytorch-classification}}}.
For VGG-16, we used the implementation from \cite{liu2016large} to allow for a
more direct
comparison\footnote{\scriptsize{\url{https://github.com/wy1iu/LargeMargin_Softmax_Loss}}}.

\textbf{VGG-FACE.} VGG-FACE is a variant of VGG optimized for face identification
\cite{Parkhi15}. In VGG-FACE, the convolutional blocks do not have BN and the
first two FC layers use Dropout with rate 0.5. We added an FC layer of size
1024, that was not present in \cite{Parkhi15}. This layer improves performance
for all tested models on Facescrub-500. We used the Caffe implementation of the
authors and fine-tune the weights provided by
them\footnote{\scriptsize{\url{http://www.robots.ox.ac.uk/~vgg/software/vgg_face/}}}. The
novel 1024 FC layer was initialized using ``Xavier'' initialization
\cite{glorot2010understanding}.

\textbf{ResNet and PreResNet.} ResNets \cite{he2016deep} are composed of
residual blocks. The concatenation of layers inside a ResNet block is
\emph{conv.-BN-conv.-BN-conv.-BN-ReLU}. The intermediate convolution layers
typically have one fourth the number of channels than the input and output
convolution layers of a block, see Table~\ref{tab:architectures}. The output of
each block is added to its input. The PreResNet architecture
\cite{he2016identity} is similar to ResNet except that inside the residual
blocks the order of the layers is inverted:
\emph{BN-conv.-BN-conv.-BN-conv.-ReLU}. No Dropout is used. For both variants we
used a publicly available implementation\footnoteref{bearpaw}.

\textbf{DenseNet.} DenseNets \cite{huang2016densely} are composed of three
DenseNet blocks. Each of these blocks is itself composed of multiple pre-BN
convolutional blocks (\emph{BN-conv.-ReLU}) with a small number of output
channels. Inside a DenseNet block, the input to each pre-BN convolutional block
is the concatenation of the output of all previous pre-BN convolutional
blocks. A transition pre-BN convolutional block is used between DenseNet
blocks. In our experiments, no bottleneck layers were used. We used Dropout of
0.2 for MNIST and SVHN, and no Dropout for CIFAR. We used a publicly available
implementation\footnote{\scriptsize{\url{https://github.com/andreasveit/densenet-pytorch}}}.
 
\subsection{Training Details}
Except for the Facescrub experiments, we always train the network from
scratch. VGG-16 uses ``MSRA'' initialization \cite{he2015delving}. For the rest
of the architectures, ``Xavier'' initialization was used
\cite{glorot2010understanding}.

In all the experiments except STL-10 and Facescrub we used SGD with Nesterov
momentum 0.9 for the optimization and batch size 64. We started with a learning
rate of 0.1 and decreased it ten-fold at 50\% and 75\% of the total training
epochs. For STL-10/Facescrub experiments, we used Adam with starting
learning rate $10^{-3}$/$10^{-5}$ and batch size 32/26. We used 164 epochs for
all architectures except for DenseNets, for which we used 300 epochs and
Facescrub where the finetuning is done for 12 epochs. The weight decay parameter
was always set to $\mu=10^{-4}$, except for STL-10+ and Facescrub, where
$\mu=10^{-3}$. Fig.~\ref{fig:convergence} shows typical convergence curves.

We implemented the OL\'E loss as a custom layer for Caffe and PyTorch. The
additional computation time is between 10\% and 33\% during training, depending
on the implementation and hardware,  because the SVD
runs on the CPU in the current implementation.

We adjusted the parameter $\lambda$ in \eqref{eq:combined_objective} with a
held-out validation set of 10\% of the training set. Note that the magnitude of the
OL\'E loss depends on the size and norm of the features matrices.  We selected
the value of $\lambda$ that produced the best result in the validation set,
averaging over 5 runs, see Fig~\ref{fig:lambda} for an example. We then
retrained the network with the entire training set and we computed the accuracy
on the test set at the end of the training.  To account for the randomness of
the training process, we repeated the training with the full training set 5
times.

\begin{figure} [t!]
  \begin{center} 
\begin{tabular}{c@{}c@{}c}  
\ \includegraphics[width=.15\textwidth]{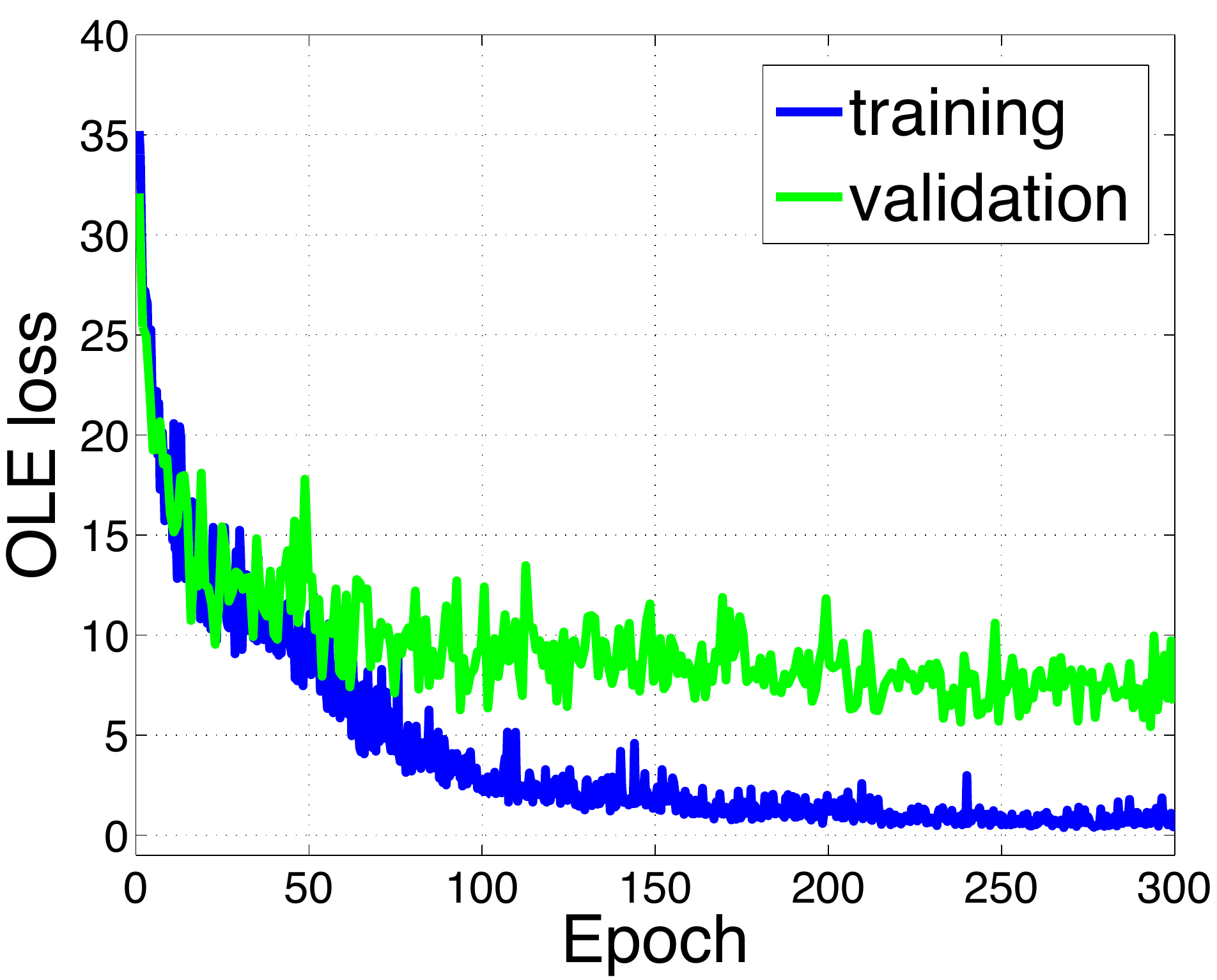}  &
 \includegraphics[width=.15\textwidth]{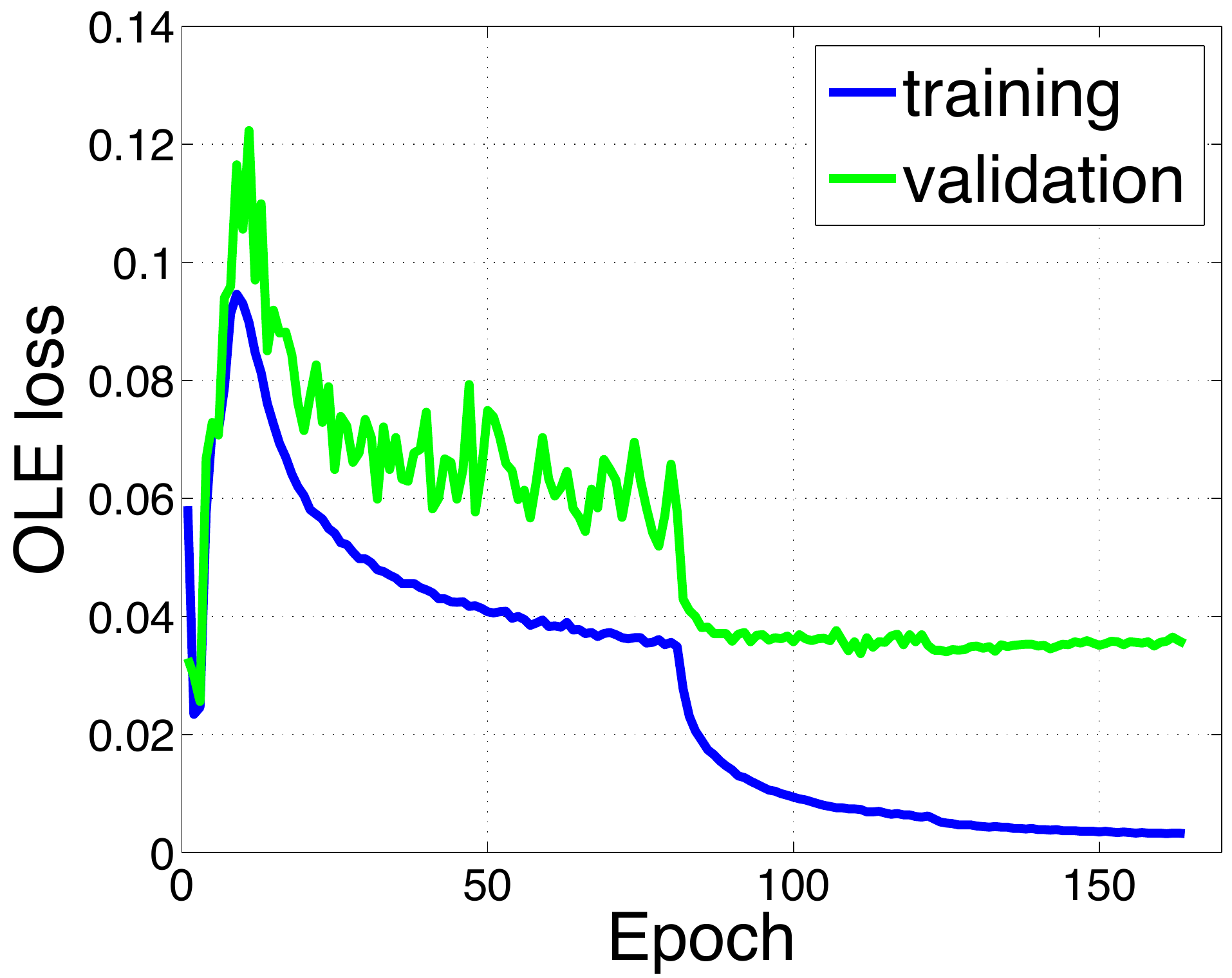}  &
 \includegraphics[width=.15\textwidth]{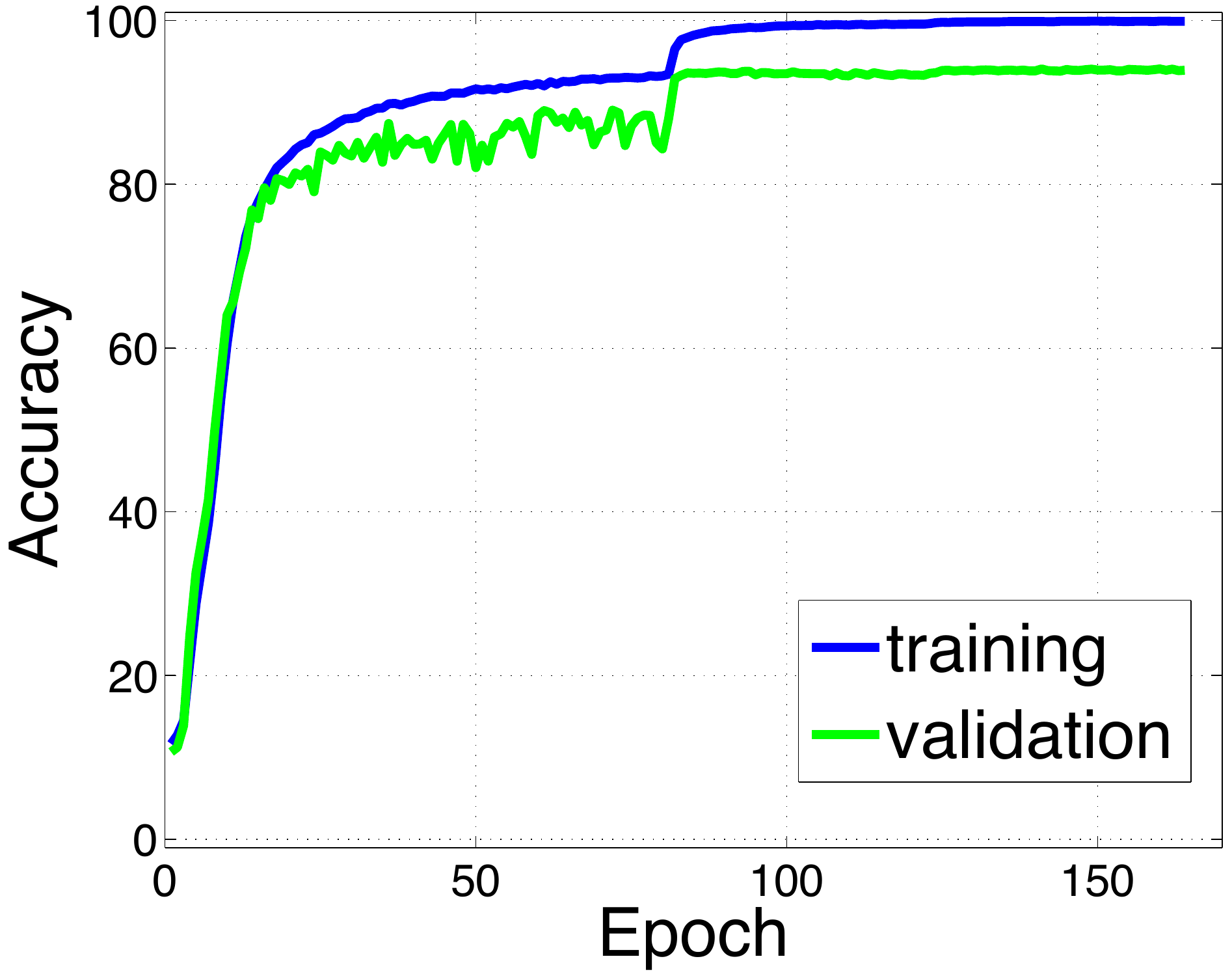}  \\
\footnotesize{(a)} & \footnotesize{(b)} & \footnotesize{(c)}
\end{tabular} 
\end{center} 
  \caption{Learning curves. \textbf{(a)} OL\'E loss when used standalone. Data
    and model from Fig~\ref{fig:toy_example}c. \textbf{(b) \& (c)} OL\'E loss and accuracy when used in combination with softmax loss for a ResNet-110 on CIFAR10+. Learning rate drops by 0.1 at 81 and 122 epochs. }
\label{fig:convergence}
\end{figure} 

\begin{figure} [t!]
  \begin{center}
\begin{tabular}{c} 
  \includegraphics[width=.20\textwidth,height=.10\textheight]{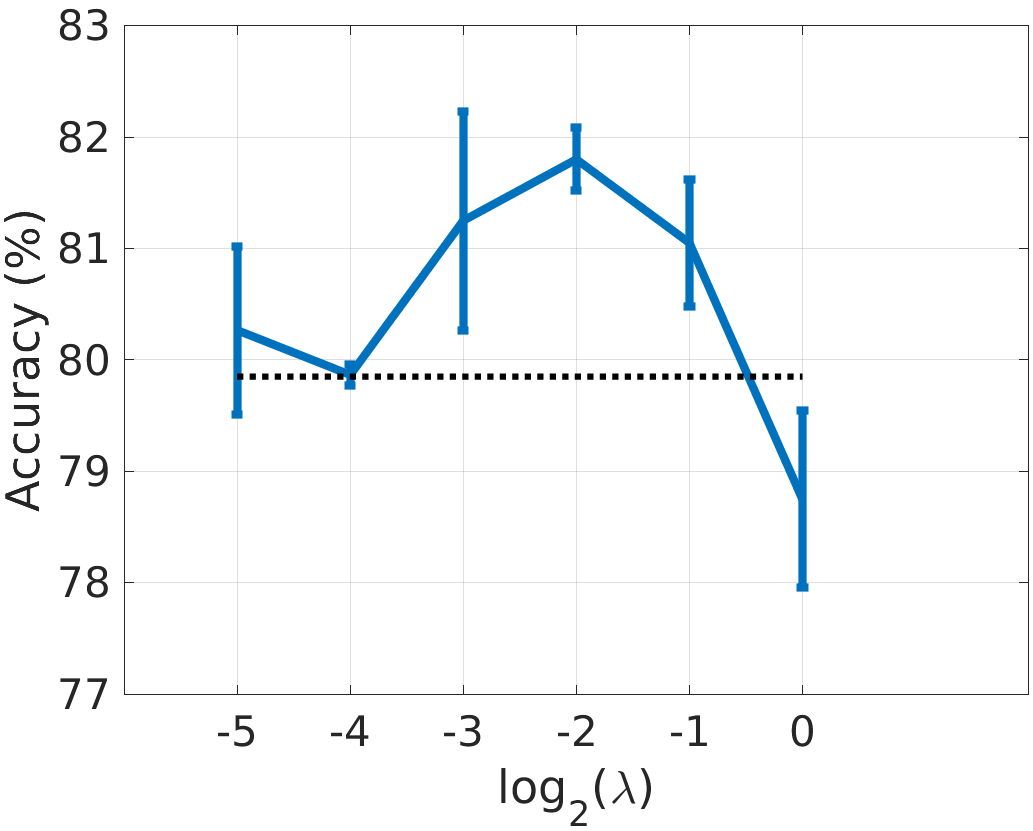}   
\end{tabular}
\end{center} 
  \caption{Validation of $\lambda$ \eqref{eq:combined_objective} for the STL-10+
    experiment. Based on this graph, we chose $\lambda=0.25$ for the final
    training. The dotted line is the average score obtained with the
    standard softmax loss.}
\label{fig:lambda}
\end{figure} 

\subsection{Visual Classification Results}
\begin{table*}[t]
\centering 
\begin{footnotesize}
\begin{tabular}{|c|c|c|c|c|c|}
\hline
Dataset & Architecture & $\lambda$ &  \% Error  ($L_o + \lambda \cdot L_s)$ & \%  Error  ($L_s$ only) &  Ref. Error (\%) \\
\hline \hline
  SVHN & DenseNet-40-12 \cite{huang2016densely} & $1/2$ & $\mathbf{3.62}\pm 0.04$ & $3.93\pm 0.08$ & $1.79$ \cite{huang2016densely} \\ 
\hline
MNIST  &  DenseNet-40-12 & $1/2$ & $\mathbf{0.78}\pm 0.04$ & $0.88\pm 0.03$ & - \\ 
\hline
 CIFAR10+ &DenseNet-40-12 & $1/8$ &  $\mathbf{5.30}\pm 0.26$ & $5.54\pm 0.13$ & $5.24$ \cite{huang2016densely}\\ 
CIFAR10+ &ResNet-110 \cite{he2016deep} &  $1/4$ & $\mathbf{5.39}\pm 0.25$ & $6.05\pm 0.8$ & $6.43$ \cite{he2016deep} \\ 
CIFAR10+ &VGG-19 \cite{simonyan2014very} &  $1/4$ & $\mathbf{7.13}\pm 0.2$ & $7.37\pm 0.11$ & - \\ 
CIFAR10+ &VGG-11 &  $1/2$ &    $\mathbf{7.73}\pm 0.14$ & $8.06\pm 0.22$ & - \\ 
\hline
CIFAR10 & VGG-16 \cite{liu2016large} & $1/2$ &   $\mathbf{7.22}\pm 0.14$ & $8.23\pm 0.13$ & $7.58$ \cite{liu2016large} \\ 
%
%
\hline
 CIFAR100+ & PreResNet-110 \cite{he2016identity} & $1/20$ & $\mathbf{22.8}\pm 0.34$ & $23.01\pm 0.19$ & $22.68\pm0.22$ \cite{he2016identity}\\ 
CIFAR100+ & VGG-19 & $1/10$ & $\mathbf{27.54} \pm 0.11$ & $28.04\pm 0.42$ & - \\ 
\hline
CIFAR100 &VGG-19 &  $1/10$ & $\mathbf{37.25}\pm 0.33$ & $38.15\pm 0.28$ &  - \\ 
\hline
FaceScrub-500 & VGG-FACE \cite{Parkhi15} &  $250$ & $\mathbf{1.55}\pm 0.02$ & $2.49\pm 0.01$ & - \\ 
\hline
 STL-10 & CNN-5 & $1/16$ & $\mathbf{25.42}\pm 0.20$ & $28.68 \pm 0.67$ & - \\ 
\hline
STL-10+ & CNN-5 & $ 1/4$ & ${\color{blue}\mathbf{16.68}}\pm 0.24$ & $18.22 \pm 0.27$ & $21.34$ \cite{thoma2017analysis} \\ 
\hline
\end{tabular}
\end{footnotesize}
\caption{Visual classification results. $L_o$ is the proposed OL\'E loss, $L_s$ is the standard softmax loss.}
\label{tab:visual_classification_results}
\end{table*}





Table~\ref{tab:visual_classification_results} shows the resulting classification
performance, with and without OL\'E. In all the experiments, we found a value of
$\lambda$ through validation such that the generalization of the network is
improved.  For reference, we include in the last column the performance
published in articles presenting the corresponding architecture for the same
datasets. Note that there could be implementation differences.

Compared to a state-of-the-art intra-class compactness method
\cite{liu2016large} using VGG-16 on CIFAR10, the lowest classification error we
obtained was $7.08\%$, compared to $7.58\%$ reported in \cite{liu2016large}.
Compared to the same network with only the standard softmax loss, a relative
reduction in the error of more than 12\% is obtained when adding the OL\'E loss.

The improvement in generalization performance is more important when only scarce
training data is available. In the Facescrub-500 experiment, where  less
than 100 samples are available per class on average, the error is dropped by
40\%. Fig.~\ref{fig:acc_vs_samples} illustrates how the advantage of using OL\'E is
more significant when fewer training data is available. We fixed $\lambda=0.25$ and
trained a CNN-5  (Table~\ref{tab:architectures}) on STL-10 without
data augmentation. We varied the number of samples from just 50 to 500
training samples per class, repeating each experiment 5 times.

In the STL-10+ experiment, the lowest classification error rate on the test set
we obtain is $16.43\%$, significantly lower than the reported state-of-the-art
error rate of $21.34\%$ in \cite{thoma2017analysis}. Note that
\cite{thoma2017analysis} uses the same training data and data augmentation
procedure.

\begin{figure} [b!]
  \begin{center}
\begin{tabular}{c} 
  \includegraphics[width=.2\textwidth]{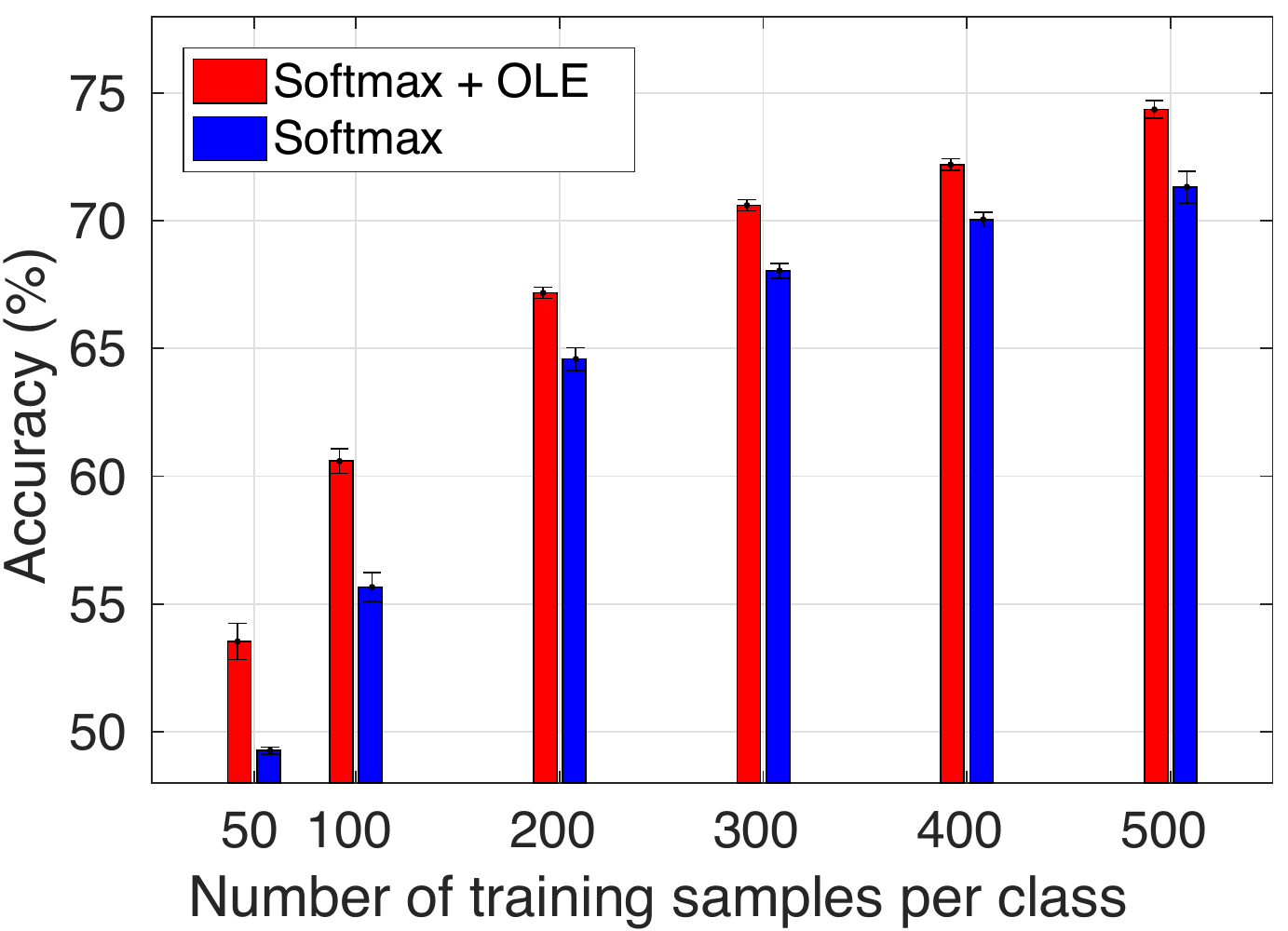}  \\  
\end{tabular}
\end{center} 
  \caption{Accuracy versus number of samples. The improved generalization when
    using OL\'E is more significant when fewer training samples are available. For
    this experiment we used STL-10 without data augmentation and we average ofer 5 runs.}
\label{fig:acc_vs_samples}
\end{figure} 

\subsection{Novelty Detection}
In this subsection we further analyze the Facescrub-500 experiment and show that
the OL\'E loss improves the novelty detection capability of the network. The
goal of novelty detection \cite{mandelbaum2017distance, pimentel2014review} is
to identify images in the test set that do not belong to any of the categories
in the training set.

Of the 530 identities in the Facescrub dataset, we took 500 identities to form
the Facescrub-500 dataset, and we left the remaining 30 identities as the novel
classes used to assess novelty detection performance. We use all the images from
the novel classes for testing (3,220 in total).

Ideally, since the novel identities are none of the known 500 subjects, their
500 class scores should all be low. We observe that this is the case when using
OL\'E, whereas when using only the softmax loss, there is typically one class
out of the known 500 that will have a confident softmax score (close to 1), see
Fig.~\ref{fig:nfa}. To show this, we varied a threshold $t\in[0,1]$ over the
softmax scores and defined the False Positive Ratio (FPR), as the number of
images of the novel classes whose softmax score is higher than $t$. In
Fig.~\ref{fig:nfa}(a) we plot the model accuracy in the 500 known subjects
against the FPR. When using the OL\'E loss, the model is able to reject most
unknown classes without significant loss of accuracy on the known
500. Fig.~\ref{fig:nfa}(b) shows the histogram of softmax scores for images of
the novel classes when using OL\'E. Most of the scores are concentrated around
$1/500$, reflecting the low confidence the OL\'E network gives to the novel
classes. On the other hand, the network trained with only the softmax loss  gives
high confidence scores to images of the novel classes, see
Fig.~\ref{fig:nfa}(c).

We verified that the OL\'E deep network did not lose face representation power
in the novel classes by running the standard verification benchmark on the
Labeled Faces in the Wild (LFW) \cite{huang2007labeled}. We observed similar AUC
($99.04\%$ vs $99.12\%$) and verification performance ($96.57\%$ vs $96.64\%$)
for the models with and without the OL\'E loss, respectively. 

\begin{figure} [t!]
  \begin{center}
\begin{tabular}{c@{}c@{}c} 
\includegraphics[width=.19\textwidth]{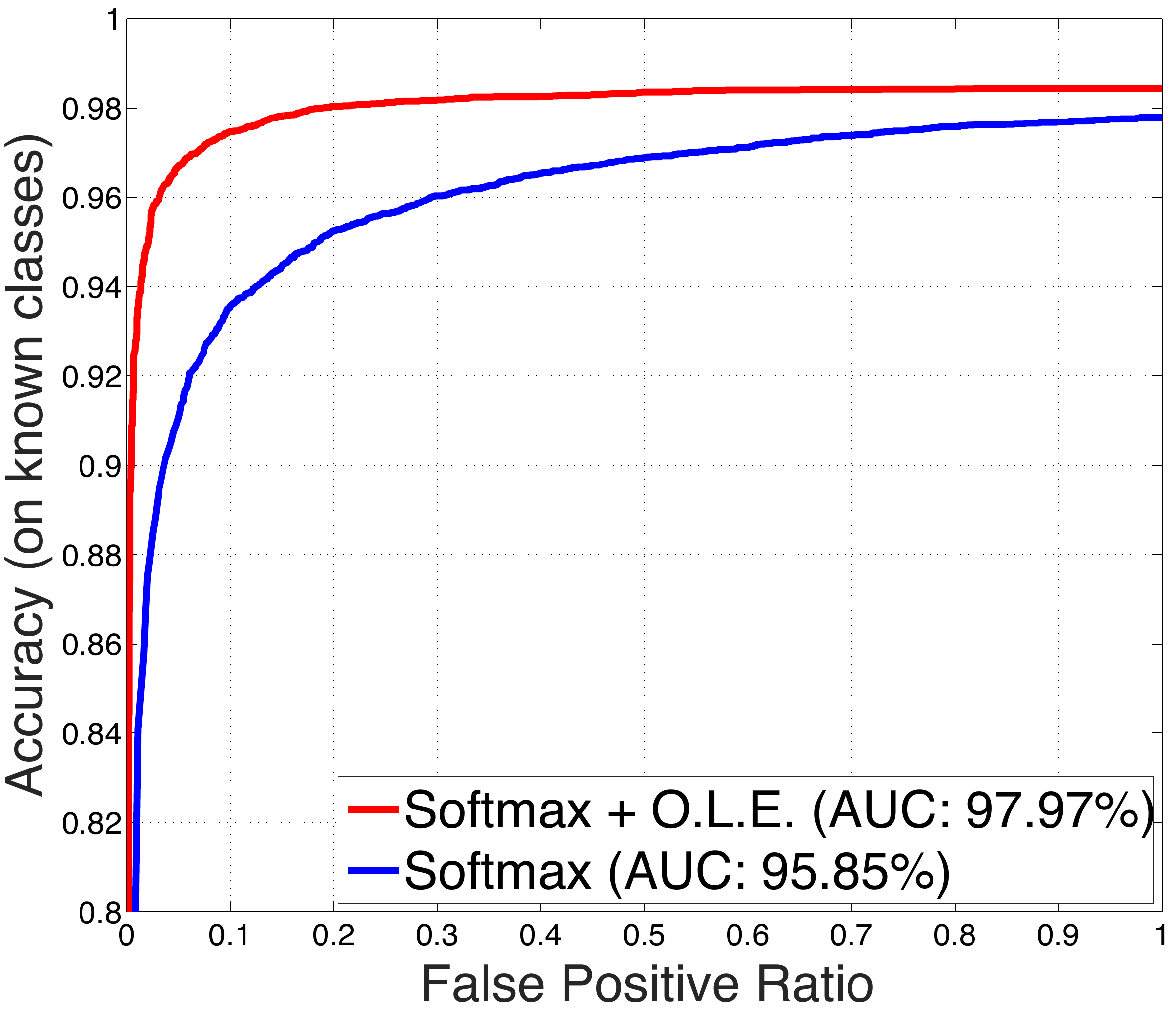}  &
 \includegraphics[width=.14\textwidth]{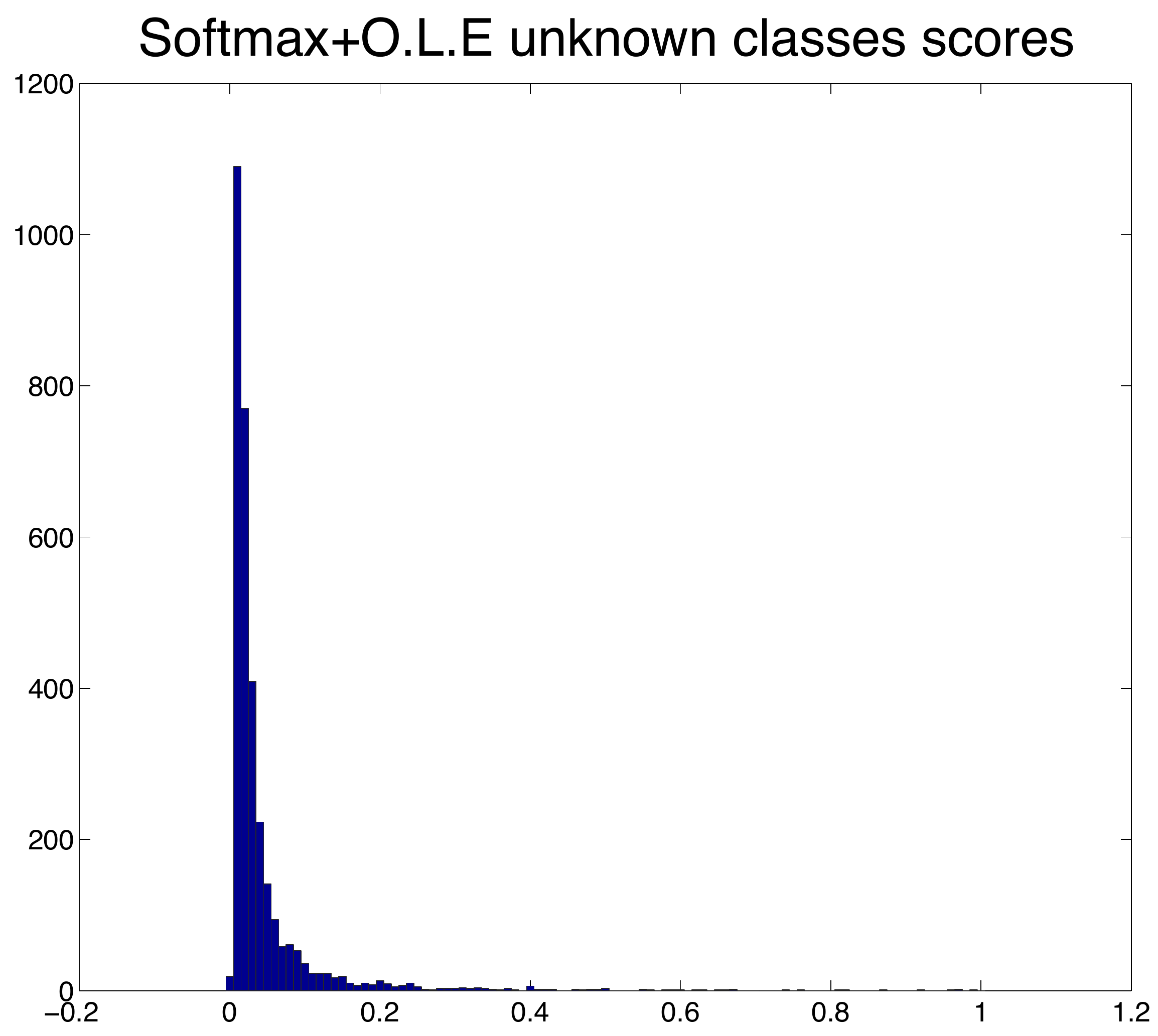} &
 \includegraphics[width=.14\textwidth]{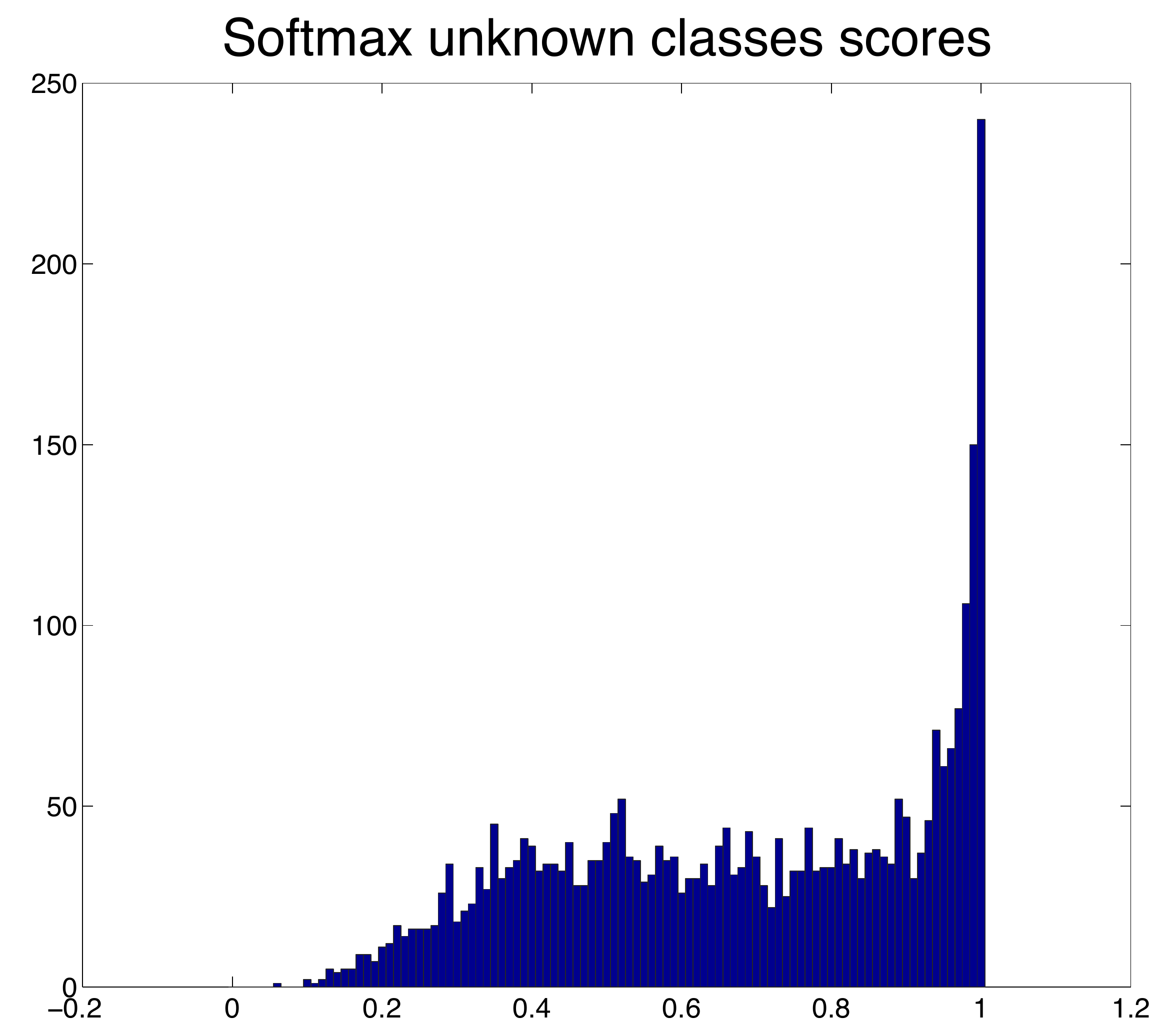}\\
\footnotesize{(a)} & \footnotesize{(b)}& \footnotesize{(c)}\\
\end{tabular}  


\end{center} 
  \caption{Application to novelty detection. \textbf{(a)} Accuracy on 500 known
    identities versus ratio of the images of the 30 novel classes that are
    wrongly classified as one of the known 500, when varying a threshold on the
    class scores. When using the OL\'E loss, more false positives can be avoided
    without losing classification performance on the known classes. \textbf{(b)
      \& (c)} Histogram of the maximum class scores for samples from the novel
    classes, with and without OL\'E, respectively. In \textbf{(b)}, scores are
    concentrated towards $1/500$, whereas in \textbf{(c)}, false high confidence
    scores are generally obtained. }
\label{fig:nfa}
\end{figure}

\subsection{Visualization of the Obtained Features} 
We illustrate the geometry of the learned deep features using OL\'E in
Fig.~\ref{fig:cifar_tsne}. In (a) and (b) we show a Barnes-Hut-SNE visualization
\cite{van2013barnes} of the obtained embedding for the validation set of
CIFAR10. The intra-class low-rank minimization reduces the
intra-class variance to only one dimension. The overall rank
maximization produces more margin (orthogonality) between classes.

In Fig.~\ref{fig:cifar_tsne} (c) and (d), we show the angle between the deep
features of (a) and (b). The 10,000 validation images are ordered by
class. With OL\'E, the relative angle is mostly 0 for images of the same
class, and 90 for images of different class. On the other hand, for the
standard softmax loss, the learned deep features have a larger intra-class
spread, and inter-class angles are not always orthogonal.

Finally, we show the spectral decomposition of the  deep feature
matrices for CIFAR10 validation set in Fig.~\ref{fig:singular_values}.  With OL\'E,
the deep features are concentrated along 10 principal dimensions, corresponding
to the learned orthogonal linear subspaces. For the
softmax loss, the deep feature matrix has its energy distributed along many
directions, reflecting the more spreading of the deep features vectors.

\begin{figure} [t!]
  \begin{center}
\begin{tabular}{c} 
 \includegraphics[width=.2\textwidth]{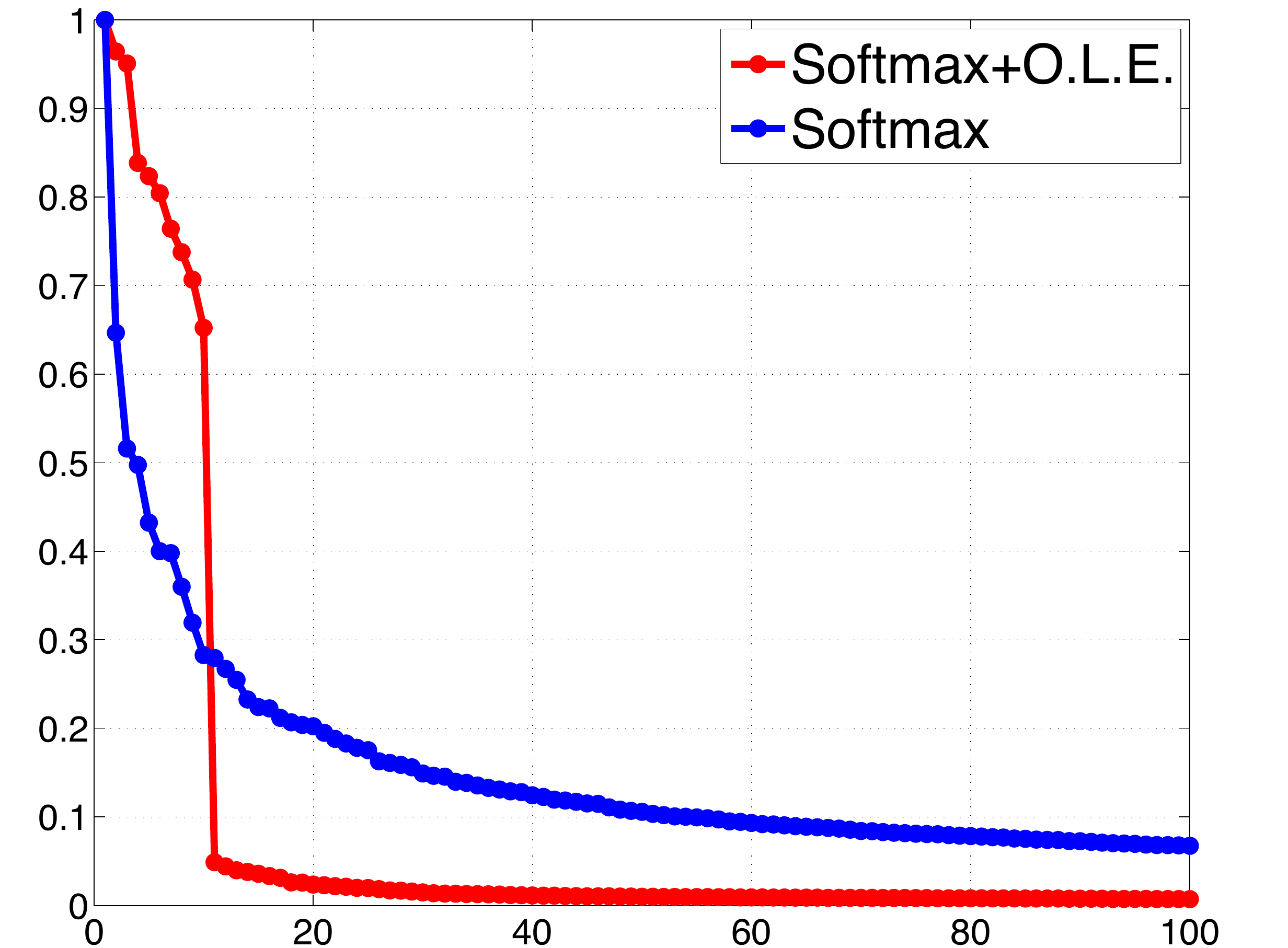} 
\end{tabular} 
\end{center} 
  \caption{Spectral analysis of the deep feature matrix obtained for CIFAR10
    validation data using VGG-16. We plot the normalized singular values of the
    feature matrix with and without OL\'E. When
    using OL\'E, the deep features are concentrated along 10 strong dimensions
    in the embedding space, corresponding to the linear subspaces where the
    features are compacted. For the standard softmax, the energy is
    distributed more evenly. }
\label{fig:singular_values}
\end{figure}

\section{Conclusions}\label{sec:conclusions}
We proposed OL\'E, a novel objective function for deep networks that
simultaneously encourages intra-class compactness and inter-class
separation of the deep features. The former is imposed as a low-rank constraint and the latter as an
orthogonalization constraint.  The proposed OL\'E loss can be used standalone as
a classification loss or in combination with the standard softmax loss for
improved performance.  We showed that  OL\'E produces more discriminative
deep networks and deep representations whose energy in the embedding space is
concentrated in a few dimensions. For classification, OL\'E is particularly
effective when  training data is scarce. Using
OL\'E, we significantly advance the state-of-the-art classification performance in the
standard STL-10 benchmark. The proposed loss introduces a new paradigm to deep
metric learning and we believe it will be a valuable tool in applications where
orthogonality in the deep representations is required.

\section*{Acknowledgements}
Jos\'e Lezama was supported by ANII (Uruguay) grant
PD\_NAC\_2015\_1\_108550. Work partially supported by NSF, DoD, NIH and Google.

\clearpage
\bibliography{article} 
\bibliographystyle{plain}

\end{document}